\documentclass[twoside,11pt]{article}

\usepackage{jmlr2e}
\usepackage{amsmath}

\jmlrheading{}{2020}{}{}{}{}{Sisi Ma and Roshan Tourani}

\ShortHeadings{Shapley Value for Model Interpretation}{Ma and Tourani}
\firstpageno{1}

\begin{document}

\title{Predictive and Causal Implications of using Shapley Value for Model Interpretation}

\author{\name Sisi Ma \email sisima@umn.edu \\
       \addr Institute for Health Informatics\\Department of Medicine\\
University of Minnesota\\ Minneapolis, Minnesota, USA
       \AND
       \name Roshan Tourani \email roshan@umn.edu \\
       \addr Institute for Health Informatics\\
University of Minnesota\\ Minneapolis, Minnesota, USA}

\editor{}

\maketitle

\begin{abstract}
Shapley value is a concept from game theory. Recently, it has been used for explaining complex models
produced by machine learning techniques. Although the mathematical definition of Shapley value is straight-forward,
the implication of using it as a model interpretation tool is yet to be described. In the current
paper, we analyzed Shapley value in the Bayesian network framework. We established the relationship between Shapley value
and conditional independence, a key concept in both predictive and causal modeling.
Our results indicate that, eliminating a variable with high Shapley value from a model do not
necessarily impair predictive performance, whereas eliminating a variable with low Shapley value
from a model could impair performance. Therefore, using Shapley value for feature selection do not
result in the most parsimonious and predictively optimal model in the general case. More importantly, Shapley
value of a variable do not reflect their causal relationship with the target of interest.
\end{abstract}

\begin{keywords}
  Causal Bayesian Networks, Predictive Models, Shapley Value, Model Explanation,  Model Interpretability
\end{keywords}

\section{Introduction}

With the increased availability of data and the rapid advancement in predictive modeling, predictive
models are playing important roles in many domains~\citep{bellazzi2008predictive,heaton2017deep,xingjian2015convolutional,sharma2011predicting}.
Many predictive modeling methods, such as deep
learning, produce highly complex models containing thousands of predictor variables. Although these models
can be highly accurate, they are also very difficult to interpret. Various methods have been
developed for model interpretation~\citep{lipton2018mythos}. Shapley value based model interpretation methods have gained a
lot of popularity recently, since they have solid theoretical foundation and has been demonstrated to
produce interpretations that matches human intuitions~\citep{vstrumbelj2014explaining,lundberg2017unified}. 
Also, since Shapley value of a variable takes into account the variables' individual as well as
combined contribution for predicting a target of interest, it is also widely adopted as a metric
for feature importance (often considered a component of model interpretation as well.). Many recent studies reports Shapley value for variables in reported predictive models ~\citep{al2020machine,dworzynski2020nationwide,artzi2020prediction}.  And, many variants
of Shapley value feature importance based feature selection methods were proposed~\citep{zaeri2018feature,cohen2007feature,mikenina1999improved,sun2012using}.

Although the definition of Shapley value is relatively straight-forward, the implication of using it
as a model explanation tool is yet to be described. Possibly because the goal of model explanation
and interpretation can be hard to define~\citep{lipton2018mythos}. Here, we focus on the relationship between
Shapley value and a variable's importance for the predictive task. We also examine whether the
Shapley value is indicative of a variable's causal relationship with the target of interest (i.e. causal interpretation). 
We seek answers to the questions of the following nature: 
If a variable have a larger Shapley value compared to other variables, 
does it mean removing it would result in a more significant performance loss?
What are the consequences of using Shapley value as a heuristic for feature selection? 
Does a variable's Shapley value indicate causal relationship with respect to the target? 

The reminder of the paper is organized as the following. In section 2, we introduce relevant concepts
for Shapley value and predictive modeling. In section 3, we
consider predictive modeling as a coalitional game and examine the characteristics of Shapley value
of each variable. We assume that the predictive models are built on data generated from a faithful
causal Bayesian network. We attempt to establish correspondence between Shapley value of variables and the
structural properties of the Bayesian network. We illustrate the implications of using Shapley value
for attributing variable importance. We also demonstrate that, in general, there is no relationship
between Shapley value and causality. The key results and practical implications for using Shapley value for model interpretation is summarized in section 4.

\section{Notations and Definitions}

In this section, we provide a minimal set of definitions of essential concepts and analytical tools used in
subsequent sections of the paper. Other relevant definitions and theorems regarding predictive modeling, variable importance, and (causal) Bayesian networks are included in Appendix A through D.
Unless specifically mentioned, we use uppercase letters to denote a variable (e.g. $X$, $Y$) and its
corresponding vertex in the causal Bayesian network underlying the data generation processes of these
variables. We use bold uppercase letters to denote a set of variables or vertices (e.g.
$\mathbf{Z}$). We use lowercase letters to denote instantiations or values of the corresponding
uppercase variables (e.g. $x$ is a instantiation of $X$). 

For all the discussion below, we use the following common notations.
$\mathbf{V}=\{V_1,V_2,\dots,V_p\}$ denotes all measured variables. $T$ denotes the target of interest
with respect to a predictive model.  $p(\mathbf{V},T)$ denotes the joint distribution over of all
variables.

The Shapley value is a solution concept in game theory~\citep{shapley1953value}. 
The Shapley value defines the division of total payoff generated by all players to individual players according to their contribution.
\begin{definition}
	Coalitional game.
	A coalitional game is a tuple $\langle\mathbf{N},v\rangle$ where $\mathbf{N}=\{1,2,\dots,n\}$
	is a finite set of $n$ players, and $v: 2^N\rightarrow{\rm I\!R}$ is a characteristic
	function such that $v(\emptyset)=0$.
\end{definition}

$v$ is a function defined over subsets of $\mathbf{N}$ that describes the value of each subset of
$\mathbf{N}$. The goal is to come up with a solution to distribute the total amount of payoff from
all players, i.e. $v(\mathbf{N})$, to each player.  We use $\phi_i(v)$ to denote the distributed
payoff to player $i$, according to the value function $v$. The Shapley value is a distribution
solution $\phi$ that have unique properties.

\begin{definition}
	Shapley value.
	For a game $\langle \mathbf{N},v \rangle$, Shapley value for a player $i$ is defined as: 
	$$
	\phi_i(v) =
	\sum_{ \mathbf{S} \subseteq \mathbf{N}-\{i\} }
	\frac{(|\mathbf{N}| - |\mathbf{S}|-1)! |\mathbf{S}|!}{ |\mathbf{N}|! }
	\left[ v(\mathbf{S} \cup \{i\}) - v(\mathbf{S}) \right].
	$$
\end{definition}
Briefly, the Shapley value of a variable $i$ is the weighted sum of the contribution of $i$ in each
subset of $\mathbf{N}$. The contribution of $i$ with respect to a subset $\mathbf{S}$ is computed by
$v(\mathbf{S}\cup\{i\})-v(\mathbf{S})$.\\

The Shapley value is considered a uniquely fair way for distributing the total payoff $v(\mathbf{N})$ into
$(\phi_1(v), \phi_2(v), \dots, \phi_n(v))$ for the $n$ players, since it satisfies the following
characteristics (in fact, Shapley value is the only solution to distribute $v(\mathbf{N})$ that
satisfies all the characteristics below):
\begin{itemize}
	\item Efficiency:
		$\sum_{i \in \mathbf{N}} \phi_i (v) = v(\mathbf{N})$.
	\item Symmetry:
		If for two players $i$ and $j$,
		$\forall \mathbf{S} \subseteq \mathbf{N}-\{i,j\}$,
		$v(\mathbf{S} \cup \{i\}) = v(\mathbf{S} \cup \{j\})$, then
		$\phi_i(v) = \phi_j(v)$.
	\item Dummy:
		If $\forall \mathbf{S} \subseteq \mathbf{N}-\{i\}$,
		$v(\mathbf{S} \cup \{i\}) = v(\mathbf{S})$, then $\phi_i(v)=0$.
	\item Additivity:
		For any pair of games $\langle \mathbf{N}, v \rangle$,
		$\langle \mathbf{N}, w \rangle$: $\phi(v+w) = \phi(v) + \phi(w)$,
		where $\forall \mathbf{S}, (v+w)(\mathbf{S}) = v(\mathbf{S}) + w(\mathbf{S})$. 
\end{itemize}

Since predictive modeling could be viewed as a coalitional game, Shapley value have been considered
as a metric for variable importance and model explanation.

\begin{definition}
	Predictive Modeling as a coalition game.
	Let $f(\mathbf{V})$ be a predictive model for a target of interest $T$. $f(\cdot)$ could be
	view as a coalition game $\langle\mathbf{V}, m\rangle$, where each variable $V_i \in
	\mathbf{V}$ is a player in the coalition game, and $m: 2^{|N|}\rightarrow{\rm I\!R}$ is a
	function that maps a subset of variables $\mathbf{S}$ to a real number representing the
	contribution of $\mathbf{S}$ for predicting $T$.
\end{definition}

There are a variety of choices for $m$, depending on nature of the task. One popular task that
Shapley value has been applied to is model explanation of individual observations~\citep{vstrumbelj2014explaining,lundberg2017unified}. 
The goal of this task is to explain why a predictive model made a specific prediction for the outcome given the
observed predictor values from an observation. In this case, $m$ is commonly defined as the deviation
of the predictive outcome from a null model (which would make prediction based solely on the
distribution of the outcome). Another task that Shapley value has been applied to is the attribution
of variable importance~\citep{owen2017shapley,shorrocks1999decomposition}. 
The goal of this task is to assign a quantitative value to each variable to
indicate their importance for the predictive performance of the model over a collection of
observations. In this case, $m$ can be based on any metric for predictive performance, such as the
area under the receiver operating curve (AUC), accuracy, sensitivity, specificity, or the mean absolute deviation of the predictive outcome from a null model. The last metric is generated by the most popular Shapley value model explanation package SHAP~\citep{SHAP}, and therefore is commonly reported in application focused literature. 
Similar to mentioned above, to define $m$, 
the predictive performance metric of choice need to be subtracted by the predictive performance of a null model, to ensure the efficiency and dummy characteristics are met.

In this study, we focus our discussion on using Shapley value for evaluating variable importance and
its implications. Our results can be extended to model explanation of individual observations, which
would be the topic of a subsequent paper.

\section{Shapley Value, Network Structure, Causation, and Prediction}

In this Section, we discuss the relationship between Shapley value and variable importance in a predictive model. We also describe the relationship between a variable's Shapley value with its (causal) structural
property with respect to the target of interest characterized by the (causal) Bayesian network over
$\mathbf{V}\cup\{T\}$.

We restrict our discussion to faithful (causal) Bayesian networks. We denote
$\langle\mathbf{V}\cup T,G,p\rangle$ a (causal) Bayesian network faithful to distribution $p$ over
variable set $\mathbf{V} \cup T$. And $f(\mathbf{V})$ is a predictive model for $T$ as a coalition
game $\langle\mathbf{V}, m\rangle$, where $m$ is maximized for each $\mathbf{S} \subseteq \mathbf{V}$
only when $p(T|\mathbf{S})$ is estimated accurately. 

\subsection{Shapley Value Summands and Conditional Independence}\label{shapsummands}
We first introduce theorems relating the summands of Shapley value of a variable to
its structural properties in the (causal) Bayesian network. This is achieved by establishing the
connection between conditional independence relations in a faithful Bayesian network, which directly
corresponds to the structural properties of the network, with the summands of Shapley value of a
variable.

Briefly, a Shapley value summand for variable $i$ given a subset $\mathbf{S}$ with respect to
predicting $T$ is $m(\mathbf{S}\cup\{i\}) - m(\mathbf{S})$; it captures the additional contribution of
$i$ to $T$ given $\mathbf{S}$.
This bears conceptual similarity to the conditional independence between $i$ and $T$ given $\mathbf{S}$
(Definition~\ref{conditionalind}).
If we constrain function $m$ to be maximized for each $\mathbf{S} \subseteq \mathbf{V}$ only when
$p(T|\mathbf{S})$ is estimated accurately
(similar to the treatment for the performance metric in Theorem~\ref{mbopt}),
then, we have: (1) the conditional independence relationship $p(T|i,\mathbf{S})=p(T|\mathbf{S})$ corresponds to
$m(\mathbf{S}\cup\{i\})-m(\mathbf{S})=0$;
And, (2) the conditional dependence relationship $p(T|i,\mathbf{S})\neq p(T|\mathbf{S})$ corresponds to
$m(\mathbf{S}\cup\{i\})-m(\mathbf{S})>0$.

\begin{theorem}\label{shapsummandPC}
	$\forall \mathbf{S} \subseteq \mathbf{V}-\{X\}$, $m(\mathbf{S}\cup X)-m(\mathbf{S})>0$ , iff
	$X\in PC(T)$. (i.e.  All Shapley value summands of parents and children of $T$ are none-zero.)
\end{theorem}
\begin{proof}
We prove $\forall \mathbf{S} \subseteq \mathbf{V}-\{X\}$,
$m(\mathbf{S}\cup X)-m(\mathbf{S})>0 \Leftarrow X\in PC(T)$.
And note that each step of the proof is reversible, so we omit the proof for the sufficiency.

Since $X \in PC(T)$, there is an edge between $X$ and $T$.
From Theorem~\ref{edgeindependence}, we know that,
$X \not\perp T|\mathbf{S}$, $\forall \mathbf{S}\subseteq\mathbf{V}-\{X\}$.
Therefore, according to 
Definition~\ref{conditionalind},
$\forall \mathbf{S}$,  $p(T | X, \mathbf{S}) \neq p(T|\mathbf{S})$ where $p(\mathbf{S})>0$,
i.e. $X$ gives more information regarding $T$ in addition to any $\mathbf{S}$.
And since we constrain $m$ to be functions that are
maximized for each $\mathbf{S}\subseteq \mathbf{V}$ only when $p(T|\mathbf{S})$ is estimated accurately,
$\forall \mathbf{S} \subseteq \mathbf{V}-\{X\}$, $m(\mathbf{S}\cup \{X\})-m(\mathbf{S})>0$. 
\end{proof}

%

\begin{corollary}\label{shapsummandnotmb}
	Not all Markov boundary members (strongly relevant variables) of $T$ have all non-zero Shapley value summands.
\end{corollary}
\begin{proof}
	Let $X$ be a parent of the children of $T$, and $X \not\in PC(T)$.
	$X$ is a Markov boundary member of $T$ but it has at least one Shapley value summands being zero
	due to Theorem~\ref{shapsummandPC}.
\end{proof}

\begin{theorem}\label{shapsammandirr}
	$\forall X$ that is not connected to $T$ by any path,
	all summands of its Shapley value $m(\mathbf{S}\cup \{X\}))-m(\mathbf{S})$ are zero,
	$\forall \mathbf{S} \subseteq \mathbf{V}-\{X\}$. 
\end{theorem}

\begin{proof}
	Given Theorem~\ref{irrnopath}, $X$ is an irrelevant variable.
	Given Theorem~\ref{edgeindependence}, $\forall \mathbf{S}\subseteq\mathbf{V}-\{X\}$,
	$X \perp T|\mathbf{S}$, i.e. $\forall\mathbf{S}$, $p(T|X,\mathbf{S})=p(T|\mathbf{S})$
	where $p(\mathbf{S})>0$.
	Therefore,
	$\forall \mathbf{S}\subset\mathbf{V}-\{X\}, m(\mathbf{S}\cup -\{X\})-m(\mathbf{S})=0$.
\end{proof}
Theorem~\ref{shapsammandirr} is related to the dummy characteristics of the Shapley value.

The theorems in this section illustrate that given faithfulness and specific $m$ functions, the
Shapley value summands contain rich information about the structural properties of
a variable with respect to a target in a (causal) Bayesian network. This connection between Shapley
value summands and the structural property of the Bayesian network can be exploited in two ways. 
Firstly, examining the distribution of the Shapley summands values can give
insight into the structural property of a variable. For example, even one summand of Shapley value
being zero indicates that the variable is not directly connected to (or is not a direct cause or
direct effect in a causal Bayesian network) of the target. Secondly, computing Shapley value exactly is often
computationally intensive, since all $2^{|\mathbf{V}|}$ summands associated with $2^{|\mathbf{V}|}$
models need to be computed. Therefore, the Shapley value is often computed by sampling from all
possible subsets $\mathbf{S}$ \citep{vstrumbelj2014explaining}. Given the structure of a
faithful Bayesian network, one might be able to sample $\mathbf{S}$ more efficiently
and reduce the computation needed to achieve a
good approximation of the Shapley value. 


\subsection{Shapley Value of Variables in a Faithful Bayesian Network}
When using the Shapley value for model explanation or feature importance, the summands of a variable
are not examined individually. But rather, the resulted sum, i.e. Shapley value, for individual
variables are compared. The theorem below describe the relationship between the Shapley
value of a variable and the variable's structural property in a faithful (causal) Bayesian network. Specifically, when variable $V_j$ d-separates $V_i$ from $T$, $V_j$'s Shapley value is larger than that of $V_i$.

\begin{theorem}\label{cindshap}
	 If $V_i \perp T | V_j$ and $V_j \not\perp T | V_i$  in a faithful (causal) Bayesian network $\langle \mathbf{V}\cup T, G, p\rangle$,	 
	 then the Shapley value of $V_j$ is larger than the Shapley value of $V_i$, i.e. $phi_{j}>phi_{i}$. 
\end{theorem}

\begin{proof}
	First we separate the summation in Shapley value of $V_i$ to two components,
	$\phi_{i|\hat{j}}$,
	the sum over sets $\mathbf{S}$ where $V_j \notin \mathbf{S}$,
	and
	$\phi_{i|j}$,
	 the sum over sets $\mathbf{S}$ where $V_j \in \mathbf{S}$,
	$$
	\phi_i=
	\sum_{\mathbf{S}\subseteq\mathbf{V}-\{V_i\}}
	\frac{(|\mathbf{N}|-|\mathbf{S}|-1)!|\mathbf{S}|!}{|\mathbf{N}|!}
	\left[ m(\mathbf{S}\cup\{V_i\}) - m(\mathbf{S}) \right]
	=
	\phi_{i|\hat{j}} + \phi_{i|j},
	$$
	$$
	\phi_{i|\hat{j}}
	=
	\sum_{\mathbf{S}\subseteq\mathbf{V}-\{V_i\}, V_j \notin \mathbf{S}}
	\frac{(|\mathbf{N}|-|\mathbf{S}|-1)!|\mathbf{S}|!}{|\mathbf{N}|!}
	\left[ m(\mathbf{S}\cup\{V_i\}) - m(\mathbf{S}) \right]
	$$	
	$$
	\phi_{i|j}
	=
	\sum_{\mathbf{S}\subseteq\mathbf{V}-\{V_i\}, V_j \in \mathbf{S}}
	\frac{(|\mathbf{N}|-|\mathbf{S}|-1)!|\mathbf{S}|!}{|\mathbf{N}|!}
	\left[ m(\mathbf{S}\cup\{V_i\}) - m(\mathbf{S}) \right]
	$$
	Similarly, for variable $V_j$, we write
	$\phi_j = \phi_{j|\hat{i}} + \phi_{j|i}$ where $\phi_{j|\hat{i}}$ only sums over
	$\mathbf{S}$ where $V_i \notin \mathbf{S}$ and $\phi_{j|i}$ only sums over
	$\mathbf{S}$ where $V_i \in \mathbf{S}$.

	Now we rewrite all the sums with $\mathbf{S} \subseteq V-\{V_i, V_j\}$,
	\begin{align*}
		\phi_i =
		\sum_{\mathbf{S} \subseteq \mathbf{V}-\{V_i, V_j\}} &
		\left\{
			\frac{(|\mathbf{N}|-|\mathbf{S}|-1)!|\mathbf{S}|!}{|\mathbf{N}|!}
			\left[ m(\mathbf{S}\cup\{V_i\}) - m(\mathbf{S}) \right]
			+
			\right. \cr & \left.
			+
			\frac{(|\mathbf{N}|-|\mathbf{S}|-2)!(|\mathbf{S}| + 1)!}{|\mathbf{N}|!}
			\left[ m(\mathbf{S} \cup \{V_i, V_j\}) - m(\mathbf{S} \cup \{ V_j \}) \right]
		\right\},
	\end{align*}
	\begin{align*}
		\phi_j =
		\sum_{\mathbf{S} \subseteq \mathbf{V}-\{V_i, V_j\}} &
		\left\{
			\frac{(|\mathbf{N}|-|\mathbf{S}|-1)!|\mathbf{S}|!}{|\mathbf{N}|!}
			\left[ m(\mathbf{S}\cup\{V_j\}) - m(\mathbf{S}) \right]
			+
			\right. \cr & \left.
			+
			\frac{(|\mathbf{N}|-|\mathbf{S}|-2)!(|\mathbf{S}| + 1)!}{|\mathbf{N}|!}
			\left[ m(\mathbf{S} \cup \{V_i, V_j\}) - m(\mathbf{S} \cup \{ V_i \}) \right]
		\right\}.
	\end{align*}
	Now writing $\phi_i - \phi_j$ we can easily cancel out the $m(\mathbf{S})$ and
	$m(\mathbf{S} \cup \{V_i, V_j \})$ terms,
	and write $\phi_i - \phi_j$ as,

\begin{align*}
		\phi_i - \phi_j
		=
		\sum_{\mathbf{S} \subseteq \mathbf{V}-\{V_i, V_j\}}
		&
		\left(
		\frac{(|\mathbf{N}|-|\mathbf{S}|-1)!|\mathbf{S}|!}{|\mathbf{N}|!}
		+
		\frac{(|\mathbf{N}|-|\mathbf{S}|-2)!(|\mathbf{S}| + 1)!}{|\mathbf{N}|!}
		\right) \cr
		&\times
		\left[ m(\mathbf{S}\cup\{V_i\}) - m(\mathbf{S} \cup \{V_j\}) \right],
\end{align*}
Since $V_i \perp T | V_j$, i.e. $p(T|\mathbf{S},V_i,V_j)=p(T|\mathbf{S},V_j)$, we have
$m(\mathbf{S}\cup\{V_i,V_j\})- m(\mathbf{S}\cup\{V_j\})=0$. And $V_j \not\perp T | V_i$, i.e.
$p(T|\mathbf{S},V_i,V_j)\neq p(T|\mathbf{S},V_i)$, we have: $m(\mathbf{S}\cup\{V_i,V_j\})-
m(\mathbf{S}\cup\{V_i\})>0$.
Therefore, $m(\mathbf{S}\cup\{V_i\})-m(\mathbf{S}\cup\{V_j\})<0$, and thus, $\phi_j > \phi_i$.
\end{proof}

Theorem~\ref{cindshap} state that, in a faithful (causal) Bayesian network, if $V_j$ renders $V_i$
conditional independent of $T$, i.e. $V_j$ contains all information in $V_i$ regarding $T$, the
Shapley value of $V_j$ is larger than that of $V_i$. It is worth noting that, the converse is true
only under special conditions such as when the graph is a chain ($V_1\rightarrow V_2\rightarrow ...
\rightarrow V_i \rightarrow... \rightarrow V_j \rightarrow ...\rightarrow T$). In general, the magnitude
of the Shapley value of two variables do not entail their conditional independence relationship.
It is easy to consider an example where $V_i$ and $V_j$ are both parents of the target $T$, one of
them could have a larger Shapley value, but none of them is conditionally independent of $T$ given
the other. The fact that one can not infer conditional independency from Shapley value in the general case
is related to our discussion in Section~\ref{shapsummands}. Specifically, the conditional independence relationship corresponds to the summands of Shapley value. 
Summing the summands together results in information loss with respect to conditional independency.

\subsection{Shapley Value of Variables: Implications for Prediction and Causation}
In this section, we introduce two theorems that have important implications for using Shapley
value for feature importance and model explanation. We focus on the Shapley values of the Markov boundary members of $T$, since these variables are of critical importance for both prediction and causation. Predictively, the Markov boundary of $T$ is the most parsimonious set of variables that contains all the information regarding $T$ and thus the optimal predictor set of $T$ (Theorem~\ref{mbopt}). Causally, the faithful causal Bayesian networks, the Markov Boundary of $T$ constitutes of the direct causes, direct effects, and direct cause of direct effects of $T$ (Theorem~\ref{MBconstitution}).

\begin{theorem}\label{MBMsmall}
	Markov boundary members of $T$ can have smaller Shapley value compared to non-Markov boundary members.
\end{theorem}

\begin{proof}
We proof the theorem by an example:
Let $G$ be a faithful Bayesian network containing 5 vertices, where $A\rightarrow T$, $B\rightarrow
T$, $C\rightarrow T$, $A\rightarrow S$, $B\rightarrow S$, $C\rightarrow S$. The data generating
process is the following:
$$A \sim \mathcal{N}(0, 2^{2})$$
$$B \sim \mathcal{N}(0, 2^{2})$$
$$C \sim \mathcal{N}(0, 2^{2})$$
$$T = A+B+C+\mathcal{N}(0, 2^{2})$$
$$S = A+B+C+\mathcal{N}(0, 2^{2})$$
We use a linear regression model and the ordinary $R^2$ as the $m(\cdot)$ function. We have: 
$$m(\emptyset)=0$$ 
$$m(\{A\})=m(\{B\})=m(\{C\})=\frac{1}{4}$$ 
$$m(\{S\})=\frac{9}{16}$$
$$m(\{A,B\})=m(\{B,C\})=m(\{A,C\})=\frac{1}{2}$$
$$m(\{A,S\})=m(\{B,S\})=m(\{C,S\})=\frac{7}{12}$$
$$m(\{A,B,S\})=m(\{A,C,S\})=m(\{B,C,S\})=\frac{5}{8}$$
$$m(\{A,B,C\})=m(\{A,B,C,S\})=\frac{3}{4}$$\\
Applying Shapley value formular, we have:
$$\phi_A=\phi_B=\phi_C=95/576=0.1649\dots$$
$$\phi_S=49/192=0.2552\dots$$
The Markov Boundary members $A$, $B$, and $C$ all have smaller Shapley value compared to non-Markov
boundary member $S$.
\end{proof}

\begin{theorem}\label{mbsumsmall}
	The sum of the Shapley values of all Markov boundary members of $T$ can be smaller than the Shapley value of a non-Markov boundary member.
\end{theorem}
\begin{proof}
We prove the theorem by an example:
Let $G$ be a faithful Bayesian network containing 3 vertices, where  $C\rightarrow A$, $C\rightarrow B$, $A\rightarrow T$, $B\rightarrow T$. The data generating
process is the following:\\
$$P(C=1)=P(C=2)=P(C=3)=P(C=4)$$
$$P(A=1|C=1)=0.05; P(B=1|C=1)=0.05$$
$$P(A=1|C=2)=0.05; P(B=1|C=2)=0.95$$
$$P(A=1|C=3)=0.95; P(B=1|C=3)=0.05$$
$$P(A=1|C=4)=0.95; P(B=1|C=4)=0.95$$
$$P(T=1|A=0,B=0)=0.9$$
$$P(T=1|A=0,B=1)=0.05$$
$$P(T=1|A=1,B=0)=0.15$$
$$P(T=1|A=1,B=1)=0.9$$

We compute the Shapley value for the best possible model (the model that results in irreducible error), such classifier will predict $T=1$ if $P(T=1|\mathbf{S}=\mathbf{s})>0.5$, and predict $T=0$ otherwise. We use the $m$ function: $\sum_\mathbf{s} P(\mathbf{S}=\mathbf{s})\max (P(T=1|\mathbf{S}=\mathbf{s}),P(T=0|\mathbf{S}=\mathbf{s})).$\\
We have:\\
$$m(\emptyset)=0.5; m(\{A\})=m(\{B\})=0.0525; m(\{C\})=0.8235$$
$$m(\{A,B\})=0.9; m(\{A,C\})=m(\{B,C\})=0.8597;$$
$$m(\{A,B,C\})=0.9;$$
and the Shapley values are:
$$\phi_A=\phi_B=0.0903\dots; \phi_C=0.2194\dots$$
We have $(\phi_A+\phi_B)<\phi_C$, i.e. the sum of the Shapley values of all Markov boundary members
of $T$ can be smaller than that of a non-Markov boundary member.
\end{proof}

	Theroem~\ref{MBMsmall} and ~\ref{mbsumsmall} have several implications for using Shapley values in
	predictive modeling. 
	With respect to feature selection, a non-zero Shapley value does not indicate a variable is 	
	non-redundant. 
	The Shapley value of a non-Markov boundary member $C$ in the proof for Theorem~\ref{mbsumsmall}, is redundant for predicting $T$ with the presence of $A$ and $B$, but its Shapley value is non-zero and larger than the Markov boundary members. Similarly, a smaller Shapley value does not indicate a variable is redundant. Several
		recent publications uses Shapley value for feature selection	 
		\citep{zaeri2018feature,cohen2007feature,mikenina1999improved,sun2012using}. 
		Typical methods either select a set of variables of a fixed size with the highest Shapley value,
		or employ recursive feature elimination using Shapley value as the method for ranking
		variables. Only selecting the top ranked several variables could potentially miss a
		Markov boundary member and results in suboptimal performance. Whereas, the recursive feature
		elimination by Shapley value could result in optimal feature set with redundant
		features, since any feature that have higher Shapley value than the Markov boundary
		member with the lowest Shapley value would be kept.
	With respect to model interpretation, the magnitude of Shapley value of variables do not necessarily 	
		correspond to causality. In the example from Theorem~\ref{MBMsmall}, 
		$S$ is neither a cause nor an effect of $T$,
		but it has larger Shapley value compared to all other causes of $T$.
		Moreover, the magnitude of Shapley value of variables also do not necessarily
		correspond to local causality: In the example from Theorem~\ref{mbsumsmall}, $C$ is an
		indirect cause $T$, but it has larger Shapley value compared to the sum of all direct
		causes of $T$\footnote{In general, the relationships learned by supervised learning predictive learning 
		algorithms are	not  guaranteed  to  reflect  causal  relationships. Therefore, there is no reason
		to believe the explanation of these models via Shapley values or other methods should reflect causality.}.

\section{Discussion and Conclusion}
In this study, we established the theoretical background for examining the Shapley value and relating
it to faithful (causal) Bayesian network. This approach revealed several implications for using
Shaplely value for variable importance and model explanation. 
We summarize the key points from this study:
\begin{itemize}
\item The summands of Shapley value correspond to conditional independency under specific definition of $m$.
\item Using Shapley value for feature selection do not guarantee obtaining the minimal optimal feature set. 
\item Magnitude of Shapley value of variables do not necessarily correspond to causality.
\item Variables that are in the local causal neighborhood of $T$ do not necessarily have larger Shapley value compared to other variables.
\end{itemize} 

Despite the theoretical importance of Theorem ~\ref{MBMsmall} and Theorem ~\ref{mbsumsmall}, they are possibility results proven by examples. We intend to explore the following questions in the future: What are the general conditions under which the disassociation between Shapley value and causality emerges? 
How prevalent is the disassociation between Shapley value and causality in datasets collected from different domains? How does the performance of Shapley value based feature selection methods compare to other types of feature selection methods?



\newpage

\appendix

\section*{Appendix A: Predictive Models and Variable Importance}

First, we introduce a set of concepts that are related to predictive modeling. We focus our
discussion on the importance of variables in the predictive model, since one of the utility of
Shapley value is to access variable importance. 
\begin{definition}
	Optimal variable set~\citep{tsamardinos2003towards}.
	Given a data set $D$ sampled from $p(\mathbf{V},T)$, a
	learning algorithm $\mathbb{L}$, and a performance metric $\mathbb{M}$, variable set
	$\mathbf{V}_{\rm opt}$ is an optimal variable set of $T$ if applying $\mathbb{L}$ on
	$\mathbf{V}_{\rm opt}$ maximizes the performance metric $\mathbb{M}$ for predicting $T$.
\end{definition}
To put it plainly, the optimal variable set $\mathbf{V}_{\rm opt}$ for predicting $T$ is a subset of
variables that produces a model that maximizes the predictive performance $T$. This definition is
intuitive but not very useful, since it neither provides the mathematical characteristics of
$\mathbf{V}_{\rm opt}$ nor a way to identify $\mathbf{V}_{\rm opt}$.
Therefore, next we introduce the Markov Blanket theory and related concepts, 
since it describes the statistical characteristics of the optimal variable set and
provides the theoretical foundation for deriving the optimal variable set.\\

We first define conditional independence, a key concept underlying variable importance.
\begin{definition}\label{conditionalind}
	Conditional Independence~\citep{pearl2009causality}.
	Let $\mathbf{V}$ be a set of variables and $p(\mathbf{V})$ be the
	joint distribution over $\mathbf{V}$,
	$\forall \mathbf{X}, \mathbf{Y}, \mathbf{Z} \subset \mathbf{V}$,
	$\mathbf{X}$ and $\mathbf{Y}$
	are conditionally independent given $\mathbf{Z}$, if
	$p(\mathbf{X} | \mathbf{Y},\mathbf{Z}) = p(\mathbf{X} | \mathbf{Z})$, where $p(\mathbf{Z})>0$.
\end{definition}

In other words, variables in $\mathbf{Y}$ do not provide additional information regarding
$\mathbf{X}$ when $\mathbf{Z}$ is available. And therefore including variables in $\mathbf{Y}$ in a
model for predicting variables in $\mathbf{X}$ is redundant, when $\mathbf{Z}$ is part of the model.
We use the symbol $\perp$ to represent independence and symbol $|$ to represent conditioning. The
statement ``$\mathbf{X}$ and $\mathbf{Y}$ are conditionally independent given $\mathbf{Z}$'' is
expressed as $\mathbf{X} \perp \mathbf{Y} | \mathbf{Z}$.

To define a variable's relevance with respect to predicting a target variable $T$, Kohavi and
John~\citep{kohavi1997wrappers} first introduced the strong, weak, and irrelevant variables.
\begin{definition}
	\label{stronglyrelevance}
	Strong relevance.
	Let $\mathbf{S}_i=\mathbf{V}-V_i$, the set of all features except $V_i$. A feature $V_i$ is
	strongly relevant to $T$ iff there exists some $v_i$, $t$, and $\mathbf{s}_i$ for which
	$p(V_i=v_i,\mathbf{S}_i=\mathbf{s}_i) > 0$, such that
	$p(T=t | V_i=v_i, \mathbf{S}_i=\mathbf{s}_i) \neq p(T=t | \mathbf{S}_i=\mathbf{s}_i)$.  
\end{definition}

\begin{definition}\label{weakrelevance}
	Weak relevance.
	A feature $V_i$ is weakly relevant to $T$ iff it is not strongly relevant, and
	$\exists \mathbf{S}_i' \subset \mathbf{S}_i$ and some $v_i$, $t$, and $\mathbf{s}_i'$ for
	which $p(V_i=v_i,\mathbf{S}_i'=\mathbf{s}_i') > 0$, such that
	$p(T=t | V_i=v_i, \mathbf{S}_i'=\mathbf{s}_i') \neq p(T=t | \mathbf{S}_i=\mathbf{s}_i')$.  
\end{definition}

\begin{definition}\label{irrelevance}
	Irrelevance.
	A feature is relevant if it is either weakly relevant or strongly relevant; otherwise, it is
	irrelevant.
\end{definition}

The strongly relevant variables contain distinct information about $T$ given all other variables,
omitting them from the model will result in suboptimal models for $T$. Weakly relevant variables do
not contain additional information about $T$ given all other variables. But they contain additional
information about $T$ given a subset of other variables. They are redundant when all the strongly
relevant variables are part of the model for $T$.The irrelevant variables do not contain any information about $T$.

The Kohavi and John definitions for variable relevance focus on the distinct information content
regarding $T$ in individual variable with respect to other variables, whereas Markov boundary and
Markov blanket considers the joint information in a set of variables with respect to the target of
interest.
\begin{definition}
	Markov blanket~\citep{pearl2009causality,aliferis2003hiton}.
	A Markov blanket $\mathbf{M}$ of the response variable $T$ in the joint probability
	distribution $p(\mathbf{V},T)$ is a set of variables conditioned on which all other variables
	are independent of $T$, that is,
	$\forall Y \in \mathbf{V}-\mathbf{M}, T \perp Y | \mathbf{M}$.
\end{definition}

\begin{definition}
	Markov boundary~\citep{pearl2009causality,aliferis2003hiton}.
	If no proper subset of Markov blanket $\mathbf{M}$ of $T$ satisfies the definition of Markov
	blanket of $T$, then $\mathbf{M}$ is a Markov boundary (MB) of $T$.
\end{definition}

The correspondence between Markov blanket and optimal feature set of $T$ was established in the
theorem below. 

\begin{theorem}\label{mbopt}
	If $\mathbb{M}$ is a performance metric that is maximized only when $p(T|\mathbf{V})$
	is estimated accurately, and $\mathbb{L}$ is a learning algorithm that can approximate any
	conditional probability distribution, then variable set $\mathbf{M}$ is a Markov blanket of
	$T$ if and only if it is an optimal feature set of $T$. Variable set $\mathbf{M}$ is a Markov
	boundary of $T$ if and only if it is a minimal optimal feature set of $T$~\citep{tsamardinos2003towards,statnikov2013algorithms}.
\end{theorem}

\section*{Appendix B: Faithful Bayesian Networks}
We next introduce a set of concepts related to Bayesian network, since we will be examining Shapley
value for predictive models built from data generated from faithful Bayesian networks. 
\begin{definition}
	Bayesian network~\citep{neapolitan2004learning}.
	Let $\mathbf{V}$ be a set of variables and $p$ be a joint probability distribution over
	$\mathbf{V}$.  Let $G$ be a directed acyclic graph (DAG) such that all vertices of $G$
	correspond one-to-one to members of $\mathbf{V}$. $\forall X \in \mathbf{V}$, $X$ is
	conditionally independent of all non-descendants of $X$, given the parents of $X$
	(i.e. Markov condition holds).
	The triplet $\langle \mathbf{V},G,p \rangle$ defines a Bayesian network
\end{definition}

\begin{definition}
	Faithfulness~\citep{spirtes2000causation}.
	If all and only the conditional independence relations that are true in the joint
	distribution $p$ are entailed by the Markov condition applied to a DAG $G$, then $p$ and $G$
	are faithful to one another.
\end{definition}

Given faithfulness, it is possible to establish the relationship between the structural property of
the Bayesian network and the statistical properties of its joint distribution. The structural property
of a Bayesian network can be described with the help of d-separation and d-connection:

\begin{definition}\label{dsep.def}
	d-separation and d-connection~\citep{spirtes2000causation}.
	A collider on a path $p$ is a vertex with two incoming edges that belong to $p$.  A path
	between $X$  and $Y$  given a conditioning set $\mathbf{Z}$ is open, if (i) every collider of
	$p$ is in $\mathbf{Z}$ or has a descendant in $\mathbf{Z}$, and (ii) no other nodes on $p$
	are in $\mathbf{Z}$. If a path is not open, then it is blocked.
	Two variables $X$  and $Y$  are d-separated given a conditioning set $\mathbf{Z}$ in Bayesian
	network iff every path between $X$ , $Y$  is blocked. if $\not \exists\mathbf{Z} \subset \mathbf{V}$ that d-separates $X$ and $Y$,
	$X$ and $Y$ are d-connected.
\end{definition}

In a faithful  Bayesian Network, d-connection and d-separation (structural properties of the
network) have a one-to-one correspondence to all conditional dependence and independence relations
(statistical properties of the joint distribution), as stated in the Theorem~\ref{dsep}:

\begin{theorem}\label{dsep}
	Two variables $X$ and $Y$ are d-separated given a conditioning set $\mathbf{Z}$ in a faithful
	Bayesian network iff $X\perp Y | \mathbf{Z}$. It follows, that if
	they are d-connected, they are conditionally dependent given $\mathbf{Z}$~\citep{spirtes2000causation}.
\end{theorem}

Below we list other two theorems that can be derived from Theorem~\ref{dsep}. They are particularly
useful when it comes to the local neighbourhood of a target variable of interest $T$.

\begin{theorem}\label{edgeindependence}
	In a faithful BN $\langle\mathbf{V},G,p\rangle$, there is an edge between the pair of nodes
	$X, Y \in \mathbf{V}$,
	iff $X \not\perp Y | \mathbf{S}$,
	$\forall \mathbf{S} \subseteq \mathbf{V} - \{X,Y\}$~\citep{spirtes2000causation}.
\end{theorem} 

\begin{theorem}\label{vstructure}
	In a faithful BN $\langle\mathbf{V},G,p\rangle$, let $PC(V_i)$ denote the set containing all
	parents of children of $V_i$. If for a triple of nodes $X, T, Y$ in $G$, $X \in PC(Y)$,
	$Y \in PC(T)$, and $X \not\in PC(T)$, then $X \rightarrow Y \leftarrow T$, iff
	$X \not\perp T| \mathbf{S}\cup{Y}$,
	$\forall\mathbf{S} \subseteq \mathbf{V}-\{X,T\}$~\citep{spirtes2000causation}.
\end{theorem}
\section*{Appendix C: Faithful Bayesian Networks and Variable Importance in Predictive Models}
Also, since the notion of strongly, weakly and irrelevant variables are based on conditional
independence, we can also infer their network structural properties given faithfulness~\citep{tsamardinos2003towards,aliferis2003hiton}.
\begin{theorem}\label{strongmb}
	In a faithful Bayesian network, a variable $X \in \mathbf{V}$ is strongly relevant if and
	only if $X \in MB(T).$.
\end{theorem}

\begin{theorem}\label{MBconstitution}
	In a faithful Bayesian network $\langle\mathbf{V},G,p\rangle$, the unique Markov boundary
	$MB(T)$ correspond to the parents, children, and parents of the children of $T$.
\end{theorem}

\begin{theorem}\label{weakpath}
	Let $\langle \mathbf{V} \cup T, G, p \rangle$ be a faithful Bayesian network. A variable $V_i
	\in \mathbf{V}$ is weakly relevant, iff it is not strongly relevant and there is an
	undirected path from $V_i$ to $T$.
\end{theorem}

\begin{theorem}\label{irrnopath}
	Let $\langle \mathbf{V} \cup T, G, p \rangle$ be a faithful Bayesian network. A variable $V_i
	\in \mathbf{V}$ is irrelevant, iff there is no path from $V_i$ to $T$.
\end{theorem}

\section*{Appendix D: Faithful Causal Bayesian Networks}
Next, we introduce causal Bayesian network. A causal Bayesian network is a Bayesian network with
causally relevant edge semantics. In a causal Bayesian network, the parents of a variable $X$ are the
direct causes of $X$, the children of $X$ are direct effects of $X$, the non-parents ancestors of $X$
are indirect causes of $X$, the non-children descendants of $X$ are indirect effects of $X$. The
causal edge semantics and causal Bayesian network can be defined as the following:
 
\begin{definition}
	Causation~\citep{pearl2009causality}.
	Let $do(X=x_i)$ denote a manipulation, where the value of $X$ is set to $x_i$.
	If $\exists x_i, x_j$,
	such that $p(Y | do(X=x_i)) \neq p(Y | do(X=x_j))$,
	then $X$ is a cause of $Y$. 
\end{definition}

\begin{definition}
	Causal Bayesian Network~\citep{pearl2009causality,spirtes2000causation}.
	A causal Bayesian network $\langle\mathbf{V}, G, p\rangle$ is the Bayesian network
	$\langle\mathbf{V}, G, p\rangle$ with the additional semantics that if there is an edge
	$X\rightarrow Y$ in $G$, then $X$ directly causes $Y$, $\forall X, Y \in \mathbf{V}$.
	(ref citation Spirtes causality book)
\end{definition}

All the theoretical results introduced in the sections above for faithful Bayesian network applies to
faithful causal Bayesian network. The main difference is, in a causal Bayesian network, structural
properties can be interpreted causally. For example, Theorem~\ref{MBconstitution} can be rewritten
for the causal Bayesian network as below.
\begin{theorem}
	Let $\langle\mathbf{V},G,p\rangle$ be a faithful causal Bayesian network. The unique Markov
	boundary $MB(T)$ corresponds to the direct causes, direct effects, and the direct causes of
	the direct effects of $T$.
\end{theorem}
For brevity, we do not restate the causal correspondence for all theorems regarding faithful Bayesian network.

\bibliography{causalbib}

\end{document}